\documentclass[twoside]{article}

\usepackage[accepted]{aistats2022}
\usepackage[round]{natbib}

\usepackage{graphicx,url,amsmath,amsthm,amssymb,mathrsfs,enumerate,subfigure,bm, cases, xcolor,breqn}
\usepackage{algpseudocode}
\usepackage{algorithm}
\usepackage{placeins}
\usepackage{bbm}

\usepackage{tikz}

\usepackage{amsbsy}
\DeclareMathAlphabet{\mathscrbf}{OMS}{mdugm}{b}{n}

\DeclareMathOperator*{\argmax}{argmax}
\DeclareMathOperator*{\E}{\mathbb{E}}
\DeclareMathOperator*{\p}{\mathbb{P}}

\newtheorem{theorem}{Theorem}
\newtheorem{definition}{Definition}
\newtheorem{lemma}{Lemma}

\newtheorem{proposition}{Proposition}

\newtheorem{assume}{Assumption}

\theoremstyle{definition}
\newtheorem{remark}{Remark}

\newcommand{\id}{{\mathbb{I}}}

\newcommand\as{{\alpha}}
\newcommand\bs{{\beta}}

\newcommand\cA{{\mathcal A}}
\newcommand\cC{{\cal C}}

\newcommand\cE{{\mathcal E}}

\newcommand\cP{{\mathcal P}}

\newcommand\cX{{[0,1]^d}}
\newcommand\cY{{[L]}}

\newcommand\N {\mathbb{N}}

\newcommand\z {\boldsymbol{z}}

\newcommand\sg {\boldsymbol{\sigma}}
\newcommand\Cr {\mathscrbf{C}_r}
\newcommand{\ind}{\perp\!\!\!\perp} 
\newcommand\ee {\varepsilon_0}
\newcommand\mh {\mathcal{M}'}
\newcommand\ms {\mathcal{M}}
\newcommand\esd {{\E\limits_{S|\z,\sg,\hat{h}}}}
\newcommand\sd {{P_{S|\z,\sg,\hat{h}}}}

\newcommand\esds {{\E_{S|\z,\sg,\hat{h}^*}}}
\newcommand\sdi {{P_{S_\cC|z_\cC,\sigma_\cC,\hat{h}}}}
\newcommand\sdif {{P_{S_\cC|z_\cC=0,\sigma_\cC,\hat{h}}}}
\newcommand\sdib {{P_{S_\cC|z_\cC=1,\sigma_\cC,\hat{h}}}}
\newcommand\psdi {{\p_{S_\cC|z_\cC,\sigma_\cC,\hat{h}}}}
\newcommand\esdi {{\E_{S_\cC|z_\cC,\sigma_\cC,\hat{h}}}}
\newcommand\sdis {{P^{\cC}_{X|\hat{h}}}}
\newcommand\sdil {{P_{Y|X,z_\cC,\sigma_\cC}}}
\newcommand\bcl {{\bf\mathcal{L}}}
\newcommand\cI {{\mathcal{I}}}

\begin{document}

%

%

\twocolumn[

\aistatstitle{Nuances in Margin Conditions Determine Gains in Active Learning}

\aistatsauthor{Samory Kpotufe \And Gan Yuan \And  Yunfan Zhao}
\aistatsaddress{Columbia University \And Columbia University \And Columbia University} ]

\begin{abstract}
  We consider nonparametric classification with smooth regression functions, where it is well known that notions of margin in $\E [Y|X]$ determine fast or slow rates in both active and passive learning. Here we elucidate a striking distinction between the two settings. Namely, we show that some seemingly benign nuances in notions of margin---somehow involving the uniqueness of the Bayes classifier, and which have no apparent effect on rates in passive learning---determine whether or not \emph{any} active learner can outperform passive learning rates. In particular, for \emph{Audibert-Tsybakov's margin condition} (allowing general situations with non-unique Bayes classifiers), no active learner can gain over passive learning in commonly studied settings where the marginal on $X$ is near uniform. Our results thus negate the usual intuition from past literature that active rates should generally improve over passive rates in nonparametric settings. 
\end{abstract}

\section{INTRODUCTION} \label{sec:intro}
Margin conditions, i.e., conditions quantifying the gap between class probabilities, have been known to determine the hardness of classification both in passive learning, i.e., where the learner only has access to i.i.d. data \citep{mammen1999smooth, tsybakov2004optimal, massart2006risk, audibert2007fast}, and in active learning where the learner can adaptively query labels \citep{castro2008minimax, hanneke2011rates, koltchinskii2010rademacher, Minsker12, hanneke2015minimax, wang2016noise, yan2016active, Locatelli-et-al17, locatelli2018adaptive}. Naturally, a main concern in active learning is in guaranteeing savings over passive learning, and here we show that some basic distinctions between margin conditions---seemingly having to do with the uniqueness of the Bayes classifier, and which appear to have gone un-noticed---determine whether savings are possible at all over passive rates in nonparametric settings. 

Here we consider the setting of nonparametric classification with smooth regression functions, i.e., one where $\eta_y(x) \doteq \p[Y = y|X = x]$ is $\alpha$-H\"older continuous for every label $y\in [L]$.  Two main notions of margin have appeared interchangeably in passive learning in this setting; assume $y = 1$ or $2$: 
$$ \text{(\rm i)} \ \p( |\eta_1 - \eta_2| \leq \tau) \lesssim \tau^\beta
\text{,(\rm ii)}\  \p( 0<|\eta_1 - \eta_2|\leq \tau) \lesssim \tau^\beta,$$
for some \emph{margin parameter} $\beta > 0$. Both definitions are termed \emph{Tsybakov's low noise or margin condition} without distinction in the literature. However, excluding $0$ as in (\rm ii) is more natural since any classifier $\hat h$ has the same error as Bayes in those regions where $\eta_1 = \eta_2$, i.e., where the Bayes is not unique. On the other hand, (\rm i) implies uniqueness (up to measure $0$) of the Bayes classifier, as seen by letting $\tau \to 0$. As such, (\rm ii) admits 
more general settings with non-unique Bayes, and is thus preferred in the seminal result of \cite{audibert2007fast} on margins in nonparametrics. 

Interestingly, using (\rm i) or (\rm ii), the minimax risk is the same in passive learning, e.g., $O(n^{-\alpha(\beta+1)/(2\alpha + d)})$ when $P_X$ is uniform, see \cite{audibert2007fast}. However, as we show, a sharp distinction emerges in active learning, where condition (\rm ii) leads to two regimes in terms of savings: 
 
\ $\bullet$ Under the common \emph{strong density} assumption, relaxing uniform $P_X$, no active learner can achieve a better rate---beyond constants---than the minimax passive rate (Theorem \ref{thm:lower}). In contrast, as first shown in \cite{Minsker12}, condition (\rm i) always leads to strictly faster rates than passive. 


\ $\bullet$ For general $P_X$, active learners can strictly gain over the worst case passive rate (Theorem \ref{thm:aggregation}). Our rates for (\rm ii) are then similar to those under (\rm i) shown in \cite{Locatelli-et-al17}.

Previous work in nonparametric active learning invariably adopted condition (\rm i) which makes sense in light of our results since savings cannot be shown otherwise. Our results in fact further highlight two sources of savings in active learning, owing to the distinction between the above two bulletted regimes: a), an active learner can evenly sample the decision boundary while i.i.d. samples might miss it under general $P_X$, and b), an active learner can quickly stop sampling in those regions where there is little to gain in excess error over the Bayes, having discovered a label or labels with sufficiently low excess error. Under near uniform $P_X$, the source of saving a) is gone since even i.i.d. data has good coverage of the decision boundary, while b) remains, although in a limited form: an active learner can only significantly benefit from regions of high margin, while it cannot effectively identify regions where multiple labels are nearly equivalent (e.g., non-unique Bayes) which it should in fact also give up on. 

Here we emphasize that our results do not preclude limited gains in practice under uniform $P_X$, since minimax rates fail to identify constants. In particular, we can refine the margin conditions to distinguish between regions of high margin and those with equivalent labels, and derive a refined upper-bound, under uniform $P_X$, that highlight such limited gains over passive learning (Theorem \ref{thm:aggregation_strong_density}). 

Finally, our upper-bounds are for general multi-class active learning, requiring minor modification over past algorithms \citep[e.g., those of ][]{Locatelli-et-al17}, namely additional book-keeping (Section \ref{sec:upper-bounds}), and refined correctness arguments. On the other hand, our main Theorem \ref{thm:lower} requires considerable new technicality over usual lower-bound arguments for active learning, involving careful \emph{randomization} of hard regions of space (see discussion in Section \ref{sec:lowerbound}). 

Our results leave open whether similar nuances in regimes of gain exist in parametric settings, e.g., under bounded VC classes, where many active learners have been shown to gain under \emph{sharp} margin conditions such as (\rm{i}) 
\citep{hanneke2011rates, koltchinskii2010rademacher, wang2016noise}. 

\paragraph{Paper Outline.} We start in Section \ref{sec:problemsetting} with technical setup, followed by an overview of main results in Section \ref{sec:mainresults}, and analysis in Section \ref{sec:analysis}. Due to space constraints, some proofs are relegated to the appendix.


\section{PROBLEM SETTING} \label{sec:problemsetting}

We consider a joint distribution $P_{X,Y}$ on $[0,1]^d \times [L]$, where 
we use the short notation $[L] \doteq \{1, \ldots, L\}$ for $L \in \N$.
Define the regression function $ \eta(x) \doteq  (\eta_1(x), \ldots, \eta_L(x))$ where $\eta_y(x) \doteq \p(Y=y|X=x)$ for $y \in [L]$. 
\vspace{0.1cm}
\begin{definition} \label{asmp:smoothness}
	The regression function $\eta$ is {\bf $(\lambda, \alpha)$-H\"{o}lder} continuous for some $\alpha\in (0, 1], \lambda>0$, if.:  
	$$\forall x, x' \in \cX, \quad \lVert\eta(x) - \eta(x')\rVert_\infty \le \lambda \lVert x - x'\rVert^{\alpha}_\infty\,.$$
\end{definition}

\begin{remark}
{For simplicity of presentation, we assume $\alpha\leq 1$ in Theorem~\ref{thm:aggregation_strong_density}. The case of $\alpha > 1$, can be handled simply by replacing the averaging in each cell with higher order polynomial regression (as done e.g. in Locatelli et al. 2017), but does not add much to the main message despite the added technicality.}
\end{remark}

\begin{definition}\label{def:partition} For $r = 2^{-k}$ for $k \in \N$, define the partition $\Cr$ of $[0, 1]^d$ 
as the collection of hypercubes $\cC$ of the form 
$ 
\prod_{i \in d} [(l_i - 1) r, l_i r)
$, $l_i \in [1/r]$. We call $\Cr$ a {\bf dyadic partition} at \emph{level} $r$. 
\end{definition}


{
\begin{definition} \label{asmp:strong_density}
		$P_X$ is said to satisfy a {\bf strong density condition} if there exists some $c_d > 0$ such that $\forall r \in \{2^{-k}: k \in \N\}$ and $\cC \in \Cr$ with $P_X(\cC) > 0$, we have
		\begin{align*}
			P_X(\cC) \ge c_d \cdot r^d\,.
		\end{align*}
\end{definition}
}

The condition clearly holds for $P_X = {\cal U}[0, 1]^d$, or simply has lower-bounded density, and is adapted from other works on active learning \citep{Minsker12, Locatelli-et-al17}. 

\subsection{Active Learning}

We consider active learning under a fixed budget $n$ of queries. At each sampling step, the learner may query the label of any point $x\in \text{support}(P_X)$ and a label $Y$ is returned according to the conditional $P_{Y| X = x}$. We let $S \equiv \{(X_i, Y_i)\}_{i=1}^n$ denote the resulting sample. A classifier $\hat h_n = \hat h_n (S): [0, 1]^d \mapsto [L]$ is then returned.

We evaluate the performance of an active learner by the excess risk of the final classifier $\hat{h}_n$ it outputs. Throughout the paper, we use the notation $\hat{h}$ for the active learning algorithm, and $\hat{h}_n$ for the final classifier the algorithm $\hat{h}$ returns.

\begin{definition}
	We consider the 0-1 risk of a classifier $h: \cX \mapsto \cY$, namely $R(h) \doteq \p (h(X) \neq Y)$, which is minimized by the so-called Bayes classifier $h^*(x) \in \argmax_{y}  \p(Y = y | X = x)$. The {\bf excess risk} $\cE(h) \doteq R(h) - R(h^*)$ is then given by:  
	$$\cE(h) = \E \ [\max_{y \in \cY}\eta_y(X)-\eta_{h(X)}(X) ].$$
	
\end{definition}

\subsection{Margin Assumption}

We start with a notion of \emph{soft} margin. 

\vspace{0.1cm}
\begin{definition} \label{def:margin}

 
 
Let $\eta_{(1)}\geq \cdots \geq \eta_{(L)}$ denote order statistics on $\eta_y, y \in [L]$. The {\bf margin} at $x$ is defined as $\ms(x) \doteq \eta_{(1)}(x)- \max_{y: \eta_y(x)\neq \eta_{(1)}(x)}\eta_y(x)$. In the case where $\forall y \in \cY, \eta_y(x) = 1/L $, we use the convention that $\max$ of empty set is $-\infty$ so that  $\ms(x) = \infty$. 
\end{definition}


\begin{definition} \label{asmp:class_tsy}
    $P_{X, Y}$ satisfies the {\bf Tsybakov's margin condition} (TMC) with $C_\bs > 0$, $\bs \geq 0$, if :
        \begin{align}\label{eqn:class_tsy}
            \forall \tau > 0, \quad  P_X \left(\{x: \ms(x) \leq \tau \}\right)  & \le C_{\beta} \tau^{\beta}. 
        \end{align}
\end{definition}
    The above extends TMC for $L=2$ to general $L$: when $L=2$, the margin $\ms(x) = |\eta_1(x) - \eta_2(x)|$ when $\eta_1(x) \neq \eta_2(x)$ and $\ms(x) = \infty$ when $\eta_1(x) = \eta_2(x) = 1/2$. The above thus coincides condition \rm{(ii)} of Section \ref{sec:intro}, i.e., admits non-unique Bayes as in \cite{audibert2007fast}, but here we allows general $L\geq 2$.
\section{OVERVIEW OF RESULTS} \label{sec:mainresults}

\subsection{No Gain under Strong Density Condition}\label{sec:lowerbound}
Surprisingly, under the Audibert-Tsybakov's margin condition, no active learner can gain in excess risk rate over their passive counterparts when we assume the strong density condition for $P_X$. For simplicity, we consider the binary case. 

\begin{theorem} \label{thm:lower}
	Consider a binary classification problem, i.e, $L=2$. Let $c_d,  \as \in (0,1], \lambda, \bs >0, C_\beta > 1$ with $\as\bs \le d$ and $\Xi = (c_d, \lambda, \as, C_\bs,\bs)$. Let $\cP(\Xi)$ denote the class of distributions on $\cX \times \{0,1\}$ such that:
	
	\ $\bullet$ $P_X$ satisfies a strong density condition with $c_d$;
	
	\ $\bullet$ the regression function $\eta(x)$ is $(\lambda,\alpha)$-H\"{o}lder;
	
	\ $\bullet$ $P_{X, Y}$ satisfies TMC with parameter $(\beta, C_\beta)$.

	Then, $\exists C_1>0$, independent on $n$, such that:  
	$$\inf_{\hat{h}} \sup_{P_{X, Y} \in \cP(\Xi)} \E\ \cE (\hat{h}_n) \ge C_1 n^{-\frac{\as(\bs+1)}{2\as+d}}, $$
	where the infimum is taken over all active learners, and the expectation is taken over the sample distribution, determined by $P$ and $\hat{h}$ jointly. 
\end{theorem}

Following the seminal results of \cite{audibert2007fast}, it is easy to show that a simple plug-in passive learner (e.g., a tree-based classifier) achieves the rate of $n^{-\frac{\as(\bs+1)}{2\as+d}}$ for any $P_{X, Y} \in \cP(\Xi)$.

Our main arguments depart from usual lower-bounds arguments in active learning \cite{castro2008minimax, Minsker12, Locatelli-et-al17} in that we do not work directly on constructing a suitable subset of $\cP(\Xi)$, but rather move to a larger class $\Sigma$ with non empty intersection $\Sigma_\beta$ with $\cP(\Xi)$. We then put a suitable measure on $\Sigma$ that concentrates on $\Sigma_\beta$; {importantly, this measure also encodes regions of $[0,1]^d$ where the Bayes is unique.} We then show that for any fixed sampling mechanism $\hat h$, the excess error of the classifier $\hat h_n$ is lower-bounded as in Theorem \ref{thm:lower}, in expectation under our measure on $\Sigma$, implying the statement of Theorem \ref{thm:lower} by concentration on $\Sigma_\beta$. A main difficulty remains in removing dependencies inherent in the observed sample $S$: this is done by decoupling the sampling $\hat h$ from the eventual classifier $\hat h_n$ by a reduction to simpler Neyman-Pearson type classifier $h^*_n$---with the same sampling mechanism as $\hat h$---whose error can be localized to regions of $[0, 1]^d$ and depends just on local $Y$ values, thanks to our choice of distributions in $\Sigma$ where little information is leaked across regions of space. This is all presented in Section \ref{sec:lowerboundanalysis}.


\subsection{Upper-Bounds} \label{sec:upper-bounds}

Theorem~1 indicates that the classical TMC is not enough to guarantee gains over passive learning, under strong density. Nonetheless, some gain can be shown under a refined margin condition that better isolates regions of space with unique Bayes label (Theorem \ref{thm:aggregation_strong_density}). Furthermore, under more general $P_X$, we show in Theorem \ref{thm:aggregation} that a better rate than passive can always be attained even under classical TMC. Both results are established using the same procedure, which we present first. We assume smooth $\eta$ in all that follows.

\begin{assume}\label{a:holder}
    $\eta(x)$ is $(\lambda, \as)$-H\"{o}lder for some known $\lambda > 0$, and some unknown $\as \in (0,1]$.
\end{assume}

As in prior work \cite{Minsker12, Locatelli-et-al17}, we assume access to $\lambda$ or any upper-bound thereof. 


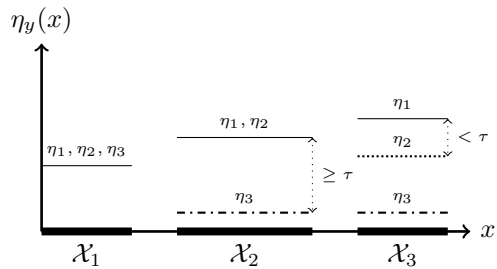
\begin{figure}
	\centering
		\begin{tikzpicture}[x=0.6 cm, y=0.25 cm]{left}
		
		  \draw[line width=3pt]  (0,0)--(2,0) node[midway, below]{$\mathcal{X}_1$};
		  \draw[line width=3pt]  (3,0)--(6,0) node[midway, below]{$\mathcal{X}_2$};
		  \draw[line width=3pt]  (7,0)--(9,0) 
		  node[midway, below]{$\mathcal{X}_3$};
	        
		  \draw(0,10)node[above]{$\eta_y(x)$};
		
		\draw [->,line width=1pt] (0,0) --(0,10)node[right] {};
		\draw [->,line width=1pt] (0,0) --(9.5,0)node[right] {$x$};
		
		\draw(3,5)--(6,5) node[midway,above]{\tiny$ \eta_1,\eta_2$};
		\draw[thick, dash dot](3,1)--(6,1) node[midway,above]{\tiny $\eta_3$};
		
        \draw[<->, dotted](6,5) -- (6,1) node[midway, right]{\tiny $ \ge \tau$};
        
		\draw(7,6)--(9,6) node[midway,above]{\tiny $\eta_1$};
		\draw[thick,densely dotted](7,4)--(9,4) node[midway,above]{\tiny $\eta_2$};
		\draw[thick, dash dot](7,1)--(9,1) node[midway,above]{\tiny $\eta_3$};

		\draw(0,3.5)--(2,3.5) node[midway,above]{\tiny $\eta_1, \eta_2, \eta_3$};
		
        \draw[<->, dotted](9,6) -- (9,4) node[midway, right]{\tiny $ < \tau$};
		\end{tikzpicture}
		\caption{Different types of margin over space.}
		\label{fig:margin}
\end{figure}


    




\subsubsection{An Adaptive Procedure}
The detailed approach is presented in Algorithm~\ref{alg:meta}, and follows an adaptation strategy of \cite{Locatelli-et-al17, locatelli2018adaptive} for unknown smoothness $\alpha$. This procedure repeatedly calls a non-adaptive subroutine, Algorithm~\ref{alg:main}, for a sequence of increasing values of $\alpha$, i.e. $\{\alpha_i\}_{i=1}^{\lfloor\log (n)\rfloor^3}$ with $\alpha_i = i/\lfloor\log (n)\rfloor^3$. 

In a departure from the binary case ($L = 2$) studied in prior work, both procedures operate by maintaining a set of candidate labels via local elimination (requiring new book-keeping), and remaining labels are then aggregated at the end to return a final classifier.


\begin{algorithm}[t]
\caption{Meta Algorithm}
\begin{algorithmic}[1]
\State Input: $n,\delta,\lambda$ 
\State Initialization: \\
$\bullet$ Set $\as_0 = 0$, $n_{0}=\frac{n}{\lfloor\log (n)\rfloor^{3}}$, $\delta_{0}=\frac{\delta}{\lfloor\log (n)\rfloor^{3}}$ \\
$\bullet$ Set minimum level $r_0 = 2^{\lfloor\log_2(n_0^{-1/d})\rfloor}$\\
$\bullet$ Set final candidate labels $\bcl_\cC = [L], \forall \cC\in \mathscrbf{C}_{r_0}$\\

\For{$i=1,...,\lfloor\log(n)\rfloor^3$}
    \State{{\color{gray} \it // Run the non-adaptive subroutine}}
    \State Set $\alpha_{i}=\frac{i}{\lfloor\log (n)\rfloor^{3}}$
    \State Run Algorithm~\ref{alg:main} with $\left(n_{0}, \delta_{0}, \alpha_{i}, \lambda,r_0\right)$ \\
    \hskip\algorithmicindent \hskip\algorithmicindent to obtain candidate labels $\{\bcl^{\as_i}_\cC\}_{\cC\in\mathscrbf{C}_{r_0}}$
    \State{ {\color{gray} // \textit{Aggregate candidate labels}}}
    \If {$\forall \cC \in \mathscrbf{C}_{r_0}, \bcl_\cC\cap \bcl_\cC^{\as_i} \neq \emptyset$}
        \State $\forall\ \cC \in \mathscrbf{C}_{r_0}$, set $\bcl_\cC =\bcl_\cC \cap \bcl_\cC^{\as_i}$
    \EndIf
\EndFor 
\State Output: $\hat h_n(x)= \min \bcl_\cC$ for $x \in \cC \in \mathscrbf{C}_{r_0}$
\end{algorithmic}
\label{alg:meta}
\end{algorithm}

Next, we discuss the non-adaptive subroutine, Algorithm~\ref{alg:main}, that assumes a known $\as$. It operates top down on dyadic partitions $\Cr$, $r=1/2 \to 0$, and aims to quickly detect cells $\cC \in \Cr$ with large sharp margin and stops sampling there; all cells with at least two remaining candidate labels are deemed \emph{active}, and form a set $\cA_r \subset \Cr$ of cells which are then refined. 

The budget is tracked throughout, by sampling as little as $n_{r,\alpha}$ points in each $\cC\in \Cr$, for
\begin{align}
\label{eqn:n_r}
n_{r,\alpha} \doteq \left. {2\log\left(\frac{2L}{\delta_0 r ^ {d+1}}\right)} \right/ (\lambda r^\alpha)^2. 
\end{align}
This sample is used to estimate $\eta$ in each cell $\cC$ as
\begin{align}\label{eqn:loss_estimator}\hat\eta_y(\cC)=n_{r,\alpha}^{-1} \sum_{i=1}^{n_{r,\alpha}} \id (Y^{\cC}_i = y),\end{align}
and eliminate labels $y$ whenever ${\hat \eta_{(1)}(\cC)}- \hat\eta_y(\cC) \ge \tau_{r,\alpha}$, where we define 
\begin{align}
\label{eqn:tau}
\hat \eta_{(1)}(\cC) \doteq \max_{y}\hat\eta_{y}(\cC), \quad  \text{ and } \tau_{r,\alpha} \doteq 6 \lambda r^\alpha. 
\end{align}

\subsubsection{Rates Under Strong Density Condition.}
We start with the following definition.  

\begin{definition}
    The {\bf sharp margin} on $\eta$ is defined as $\mh(x) \doteq \eta_{(1)}(x) - \eta_{(2)}(x)$, where we have $\eta_{(1)} = \eta_{(2)}$ when the Bayes label is not unique at $x$. 
\end{definition}

\begin{assume} \label{asmp:general_tsy}
		$P_{X, Y}$ satisfies a {\bf refined margin condition} (RMC) with 
		$\ee, C_\beta, \beta, \beta' > 0$ with $\beta' \geq \beta$: 
		\begin{align*}
			\forall \tau > 0, \quad P_X\left(\{x: \ms(x) \leq \tau \}\right)  & \le C_{\beta} \tau^{\beta};  \text{ and }  \\
			\forall \tau > 0, \quad P_X \left(\left\{x: \ms'(x) \leq \tau \right\}\right) &\le \ee + C_\beta \tau ^{\beta'}.
		\end{align*} 
\end{assume}

\begin{remark}
    The two conditions in Assumption~\ref{asmp:general_tsy} differ when the Bayes is not unique, i.e., when $\p (\ms' = 0) \doteq \ee > 0$, otherwise $\ms = \ms'$ a.e., and we may choose $\beta = \beta'$. 
    For illustration, consider the example of Figure~\ref{fig:margin} with $L =3$. We have $\{x:\ms(x) \le \tau\} = \mathcal{X}_3$, while $\{x:\mh(x) \le \tau\} = \cup_{i=1}^3 \mathcal{X}_i$. In particular, $\ee = P_X(\mathcal{X}_1 \cup \mathcal{X}_2)$, as $\mh = 0$ on $\mathcal{X}_1 \cup \mathcal{X}_2$.
\end{remark}

\begin{algorithm}[t]
\caption{Non-adaptive Algorithm}
\label{alg:main}
\begin{algorithmic}[1]
\State Input: $n_0,\delta_0,\alpha,\lambda,r_0$
\State Initialization: \\
$\bullet$ Initial level: $r=1/2$ \\
$\bullet$ Active cells: $\mathcal{A}_{r}= \mathscrbf{C}_{r}$ \\
$\bullet$ Budget up to level $r$: $m_r=|\cA_r| n_{r,\alpha}$ (see \eqref{eqn:n_r})\\
$\bullet$ Candidate labels: $\bcl^\as_{\cC}=[L], \forall \ \cC\in \mathscrbf{C}_{r}$

\While{$(m_r \leq n_0)$\text{ and }$(|\cA_r| > 0)$} 
    \State{{\color{gray}// \it Eliminate bad labels}}
    \For{each $\cC\in\mathcal{A}_r$}
        \State Samples $({X}^{\cC}_{i}, {Y}^{\cC}_{i})_{j\leq n_{r,\alpha}}$ in cell $\cC$
        \State Compute $\{\hat\eta_y(\cC)\}_{y \in [L]}$ by \eqref{eqn:loss_estimator}
        \State Set
        $\bcl^\as_{\cC}=\bcl^\as_{\cC}\setminus\{y:\hat\eta_{(1)}(\cC) - \hat\eta_{y}(\cC) \geq \tau_{r,\alpha}\}$\eqref{eqn:tau}
    \EndFor
    \State{{\color{gray}// \it Pass information to the next level}}
    \State $\forall \cC' \in \mathscrbf{C}_{r/2}$ with $\cC' \subset \cC$, set $\bcl^\as_{\cC'}=\bcl^\as_{\cC}$
    \State Set $\mathcal{A}_{r/2} = \cup \{\cC'\in \mathscrbf{C}_{r/2}:\cC'\subset \cC$ for some \\ 
    \hskip\algorithmicindent \hskip\algorithmicindent \hskip\algorithmicindent \hskip\algorithmicindent
    \hskip\algorithmicindent \hskip\algorithmicindent$\cC \in \cA_r$ with $|\bcl^\as_\cC| \ge 2$\}
    \State Set $r=r/2$ {\color{gray} \it // Go to next level}
    \State Set $m_{r/2}=m_r+|\mathcal{A}_r|n_{r,\alpha}$ {\color{gray}// \it Update the budget used}
\EndWhile  
\State Set $r_\min=2r$ {\color{gray} \it // The minimum level reached }
\State Set $\bcl^\as_\cC=\bcl^\as_{\cC'},\forall \ \cC \in \mathscrbf{C}_{r_0}$ with $\cC \subset \cC'\in \mathscrbf{C}_{r_{\min}}$
\State Output: $\{\bcl^\as_\cC\}_{\cC \in \mathscrbf{C}_{r_0}}$

\end{algorithmic}
\end{algorithm}
\FloatBarrier


The upper-bound shown in Theorem \ref{thm:aggregation_strong_density} below depends on $\ee$, and recovers existing bounds (for the binary case) when $\ee = 0$, namely $\widetilde{O}\left(n^{-\as(\bs'+1)/(2\as+d-\as\bs')}\right)$ as shown e.g. in \cite{Minsker12, Locatelli-et-al17} under sharp margin. 
This is an improvement over the passive learners, and matches the active lower-bound in \cite{Minsker12} under strong density condition with $\as\bs \le d$. For large $\ee > 0$, the first term $\widetilde{O}\left(n^{-\as(\bs+1)/(2\as+d)}\right)$ dominates, matching our lower-bound of Theorem \ref{thm:lower}. 

\begin{theorem} \label{thm:aggregation_strong_density} 
Let $n \in \N$ and { $\as \in (0,1]$}
and $\alpha\beta' \le d$.  
Let $\hat{h}_{n}$ denote the classifier returned by Algorithm~\ref{alg:meta} with input $n$, $\lambda$ and $0 < \delta < 1$. Under Assumption~\ref{a:holder}~and~\ref{asmp:general_tsy}, and assume further that strong density condition holds for some $c_d>0$, then with probability at least $1-\delta$,
\begin{align*}
 \mathcal{E} \left(\hat{h}_{n}\right)
\leq \ & C_2 \left(
\ee^{\frac{\alpha(\beta+1)}{2 \alpha+d}}  
\left(\frac{\lambda^{\frac{d}{\alpha}}\log^3(n) \log \left(\frac{4L  \lambda^{2} n}{\delta}\right)}{n}\right)^{\frac{\alpha(\beta+1)}{2 \alpha+d}}\right. \\ &+
\left.\left(\frac{\lambda^{\frac{d}{\alpha}\vee\beta'}\log^3(n) \log \left(\frac{4L  \lambda^{2} n}{\delta}\right)}{n}\right)^{\frac{\alpha(\beta'+1)}{2 \alpha+d-\alpha\beta'}} 
\right)
\end{align*}
for some constant $C_2> 0$ independent of $n,\delta,\lambda, L, \ee$.
\end{theorem}

\begin{remark}
{
The bound is trivial for $\alpha< \frac{1}{\log (n)}$, since $n^{-\alpha}\ge n^{-1/\log (n)}= \frac{1}{e}$. Thus, we only need to show for $\alpha\ge \frac{1}{\log (n)}$.}
\end{remark}

A main novelty in the analysis is to separately consider parts of space with unique Bayes, determined by $\ee$ and $\beta'$, and those parts of space where the Bayes might not be unique, but which still have margin, determined by $\beta$. Furthermore, our consideration of general multiclass, together with non-unique Bayes, brings in a bit of added technicality due largely to additional book-keeping. In particular, while in \cite{Minsker12, Locatelli-et-al17}, the main correctness argument involved showing that all labeled parts of space (i.e. cells with a single label left) have 0 excess error w.h.p., we additionally have to show that in fact, remaining labels in most active cells are close in error to Bayes. 

\subsubsection{Rates for General Densities}

For general $P_X$, on the other hand, Algorithm~\ref{alg:main} has an excess risk rate of order $\widetilde{O}(n^{-(\as(\bs+1))/(2\as+d)})$, which is always faster than the lower minimax rate ${O}(n^{-(\as(\bs+1))/(2\as+d+\as\bs)})$ for passive learning of \cite{audibert2007fast} under the same conditions. 

{In other words, under TMC, which allows non-unique Bayes classifiers, active learning guarantees savings over the worst-case rate of passive learning, given the ability to evenly sample the decision boundary. }

\begin{theorem}\label{thm:aggregation} 
Let $n \in \N$ and $\as \in (0,1]$
and $\alpha\beta' \le d$.  
Let $\hat{h}_{n}$ denote the classifier returned by Algorithm~\ref{alg:meta} with input $n$, $\lambda$ and $0 < \delta < 1$. Under Assumption~\ref{a:holder}~and~\ref{asmp:general_tsy}, with probability at least $1-\delta$,
\begin{align*}
    \mathcal{E} \left(\hat{h}_{n}\right)
    \leq& \ C_3\left(\frac{\log^3(n) \lambda^{\frac{d}{\alpha}} \log \left(\frac{4L  \lambda^{2} n}{\delta}\right)}{ n}\right)^{\frac{\alpha(\beta+1)}{2 \alpha+d}}
\end{align*}

for some constant $C_3 > 0$ that does not depend on $n,\delta,\lambda, L, \ee$.
\end{theorem}

The proof ideas follow similar outlines as for Theorem \ref{thm:aggregation_strong_density}, though more direct.

\section{ANALYSIS}\label{sec:analysis}

\subsection{Proof of Theorem~\ref{thm:lower}}\label{sec:lowerboundanalysis}


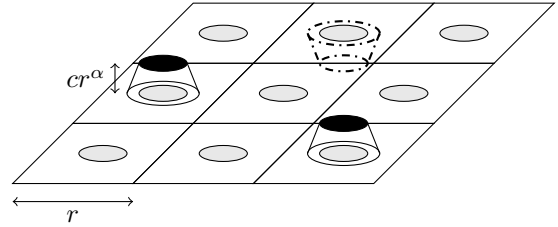
\begin{figure}[htbp]
	\centering
	\begin{tikzpicture}[x=0.8cm , y=0.8cm ]{left}
\def \nx{2}
\def \ny{2}
\foreach \x in {0,...,\nx} {
	\foreach \y in {0,...,\ny} {
	\draw(\x*2+\y,\y)--(\x*2+\y+2,\y)--(\x*2+\y+3,1+\y)--(\x*2+\y+1,\y+1)--cycle;
	\draw[fill={rgb:black,1;white,10}](\x*2+\y+1.5,\y+0.5) ellipse (0.4 and 0.133);
	}
}

\def \x{0}
\def \y{0}
\draw[<->](\x*2+\y,\y-0.3)--(\x*2+\y+2,\y-0.3) node[midway, below]{$ r$};

\def \x{2}
\def \y{0}
\def \z{1}
\draw(\x*2+\y+1.5,\y+0.5) ellipse (0.6 and 0.2);
\draw[fill](\x*2+\y+1.5,\y+0.5+0.5*\z) ellipse (0.4 and 0.133);
 \draw(\x*2+\y+1.5-0.6,\y+0.5)--(\x*2+\y+1.5-0.4,\y+0.5+0.5*\z);
 \draw(\x*2+\y+1.5+0.6,\y+0.5)--(\x*2+\y+1.5+0.4,\y+0.5+0.5*\z);

\def \x{0}
\def \y{1}
\def \z{1}
\draw(\x*2+\y+1.5,\y+0.5) ellipse (0.6 and 0.2);
\draw[fill](\x*2+\y+1.5,\y+0.5+0.5*\z) ellipse (0.4 and 0.133);
 \draw(\x*2+\y+1.5-0.6,\y+0.5)--(\x*2+\y+1.5-0.4,\y+0.5+0.5*\z);
 \draw(\x*2+\y+1.5+0.6,\y+0.5)--(\x*2+\y+1.5+0.4,\y+0.5+0.5*\z);
  \draw[<->](\x*2+\y+1.5-0.6-0.2,\y+0.5)--(\x*2+\y+1.5-0.6-0.2,\y+0.5+0.5)node[midway, left]{$c r^\alpha$};

\def \x{1}
\def \y{2}
\def \z{-1}
\draw[thick, dash dot](\x*2+\y+1.5,\y+0.5) ellipse (0.6 and 0.2);
\draw[thick, dash dot](\x*2+\y+1.5,\y+0.5+0.5*\z) ellipse (0.4 and 0.133);
 \draw[thick, dash dot](\x*2+\y+1.5-0.6,\y+0.5)--(\x*2+\y+1.5-0.4,\y+0.5+0.5*\z);
 \draw[thick, dash dot](\x*2+\y+1.5+0.6,\y+0.5)--(\x*2+\y+1.5+0.4,\y+0.5+0.5*\z);

	\end{tikzpicture}
	\caption{Construction for Theorem \ref{thm:lower} onto a partition $\Cr$ of $[0, 1]^d$, for a critical $r = r(n, \alpha, \beta, \lambda)$. Two \emph{coins} are thrown in each cell $\cC$, one $z_\cC$ with some bias determining whether the Bayes is unique, the other $\sigma_\cC$ determining the Bayes label. The regression function is constructed as $\eta_C \approx 1/2 \pm r^\alpha$, and together with $P_X$ forces any $\hat h$ to mostly rely on local information. 
	}
	\label{fig:construction}
\end{figure}


\subsubsection{Construction of Joint Distributions}\label{sec:construction_distribution}
		We again operate over a dyadic partition $\Cr$ of the unit cube $[0,1]^d$. Let $r = c_1 n^{-\frac{1}{2\as+d}}$, where $c_1 = \frac{64}{\lambda^{2}}$. Without loss of generality, we assume that $-\log_2 r \in \N$. Furthermore, we denote the barycenter of any $\cC \in \mathscrbf{C}_r$ as $x_\cC$. 
		The marginal distribution $P_X$ has the density with respect to the Lebesgues measure:
			\begin{align*}
				f(x) \doteq
				\begin{cases}
					4^d & \text{ if }\lVert x - x_\cC \rVert < r/8 \text{ for some } \cC \in \Cr;\\
					0 & \text{ otherwise.}
				\end{cases}
			\end{align*}
		where $\lVert \cdot \rVert$ is the supnorm. Let $\z= (z_\cC)_{\cC \in \Cr}\in \{0,1\}^{|\Cr|}$ and  $\sg = (\sigma_\cC)_{\cC \in \Cr} \in \{\pm 1\}^{|\Cr|}$. Define:
		$$\eta_{\z, \sg}(x) \doteq 1/2 + c_\eta \sum_{\cC \in \Cr} z_\cC\cdot  \sigma_\cC \cdot {\phi_\cC(x)},$$
		where $c_\eta = \lambda / 8$, and $$\phi_\cC(x) = \min\left\{(2r^\as - 8r^{\as-1}\lVert x - x_\cC \rVert)_+, r^\as\right\}.$$ For each pair $(\z, \sg)$, one can define a joint probability distribution $P_{\z,\sg}$ characterized by $P_X$ and $\E[Y|X=x] = \eta_{\z, \sg}(x)$. See Figure~\ref{fig:construction} for an example of $P_{\z,\sg}$ for $d = 2$ and $r = 3$. In particular, $P_X$ is uniformly distributed within its support, which is the area shaded in gray. In a cell $\cC \in \Cr$ where $z_\cC = 1$, we have a small bump in regression function, of which the direction is determined by $\sigma_\cC$. By construction, $\eta_{\z, \sg}$ is always a constant in the intersection of $\cC$ and the support of $P_X$, with the only possible values being $1/2$ and $1/2 \pm c_\eta r^{\as}$.
	\begin{remark}
	    
	    Our construction in fact satisfies the strong density assumption of \cite{audibert2007fast}: their assumption requires lower-bounded densities only on the distribution support which is allowed to be disconnected, as constructed here. 
	\end{remark}
\subsubsection{Establishing the Lower-bound}
		
		The proof of the Theorem~\ref{thm:lower} is divided and conquered by Proposition \ref{prop:subclass} to \ref{prop:NP}.
		Let $\Sigma \doteq \{P_{\z,\sg}: (\z,\sg) \in \{0,1\}^{|\Cr|} \times \{\pm 1\}^{|\Cr|}\}$ and $\Sigma_{\bs} \doteq \{P_{\z,\sg}: (\z,\sg) \in \Theta_\bs\}$ where $ \Theta_\bs \doteq \{(\z,\sg): \forall \tau > 0, P_X(\{x: 0 < |2\eta_{\z, \sg}(x) - 1| \le \tau\})\le C_\beta\tau^{\bs}\}$.   
		
		\begin{proposition}\label{prop:subclass}
		$\Sigma_{\bs} \subset \cP(\Xi)$. Consequently,
		     $$\inf_{\hat{h}} \sup_{P_{X, Y}\in \cP(\Xi)} \E \ \cE(\hat{h}_n) \ge \inf_{\hat{h}} \sup_{P_{X, Y}\in \Sigma_{\bs}} \E \ \cE(\hat{h}_n).$$
	    where the infimum is taken over all active learners.
		\begin{proof} 
		    Let $P_{\z,\sg} \in \Sigma_\bs$. The TMC is satisfied by construction, and it is trivial to show that strong density condition holds for $c_d = 1$. It is left to show that $\eta_{\z,\sg}$ is $(\lambda, \as)$-H\"{o}lder. In fact, this hold for all $P_{\z,\sg} \in \Sigma$. 
		    
		    Let $x, x' \in \cX$. If they are in a common cell $\cC$, then 
            \begin{align*}
                 |\eta_{\z,\sg}(x) -  \eta_{\z,\sg}(x')| & \le z_\cC c_\eta (8r^{\as-1} \lVert x - x'\rVert) \\
                 & \le \lambda \lVert x -x'\rVert^\as,
            \end{align*}
            where the last inequality is due to the fact $r / \lVert x - x'\rVert \ge 1$ and $\as - 1 < 0$. If they are in different cells, $|\eta_{\z,\sg}(x) -  \eta_{\z,\sg}(x')| = 0$ if $\lVert x - x'\rVert < r/4$. Therefore,
            \begin{align*}
                  |\eta_{\z,\sg}(x) -  \eta_{\z,\sg}(x')| 
                 \le 2c_\eta r^\as 
                 \le  \lambda \lVert x- x' \rVert^\as.
            \end{align*}
		    Therefore, $\eta_{\z,\sg}$ is $(\lambda,\as)-$H\"{o}lder.
		\end{proof}
		\end{proposition}
		
		Let $\z\in \{0, 1\}^{|\Cr|} \overset{\tiny \text{i.i.d}}{\sim}\text{Ber}(r^{\as\bs})$, and $\sg \in \{\pm 1\}^{|\Cr|} \overset{\tiny \text{i.i.d}} {\sim} \text{Radamacher}(1/2)$, $\z \ind \sg$. 
		
		\begin{proposition}\label{prop:sup_to_expectation}
		     Let $\hat{h}$ be any active learner. Then,
		     \begin{align*}
		          \sup_{P_{X, Y}\in \Sigma_{\bs}}{\esd}\cE(\hat{h}_n) 
		        \ge & \E_{\z,\sg}{\esd}\cE(\hat{h}_n) \\ & \quad - \exp(-c_2r^{-(d-\as\bs)}),
		     \end{align*}
		     for some $c_2 > 0$, where $\esd(\cdot)$ is expectation taken over sample $S$, under the sampling distribution $\sd$ determined by $P_{\z,\sg}$ and $\hat{h}$ jointly, and $\E\limits_{\z,\sg}(\cdot)$ is the expectation taken over ${\z, \sg}$.
		     
		\end{proposition}
		
		
		\begin{proof}
		    By construction, $|2\eta_{\z,\sg}(x) - 1|$ is either $0$ or bounded from below by $2 c_\eta r^\as$ almost surely. Thus, we only need to consider $\tau = tc_\eta r^{\as}$ for $t \ge 2$. For given $\z$, $P_X(\{x: 0<|\eta_{\z,\sg}(X)-1/2| \le t c_\eta r^\as\}) \le r^d \mathbf{1}^\top \z$. By Chernoff bound (Lemma B.1), 
	 	\begin{align*}
	 		& \p_{\z}\left(r^d \mathbf{1}^\top \z \le C_\bs r^{\as\bs}\right) 
	 		\ge  1 - \exp\left({c_2r^{-(d-\as\bs)}}\right),
	 	\end{align*}
	 	where $c_2 = (C_\bs-1)^2/3$. Therefore, $\p\limits_{\z,\sg}((\z,\sg) \in \Theta_{\bs}) \ge 1 - \exp(c_2 r^{-(d-\as\bs)})$ and 
	 	\begin{align*}
	 	    \sup_{P_{\z,\sg} \in \Sigma_{\bs}}\E\cE(\hat{h}_n) 
	 	    \ge &\E_{{\z,\sg}}\left[\left.{\esd}\cE(\hat{h}_n)\right|(\z,\sg) \in \Theta_{\bs}\right]\\
	 	    \ge & \E_{{\z,\sg}}{\esd}\cE(\hat{h}_n) - \p_{\z,\sg}((\z,\sg) \not\in \Theta_{\bs}) \\
	 	    \ge & \E_{{\z,\sg}}{\esd}\cE(\hat{h}_n) - \exp(c_2r^{-(d-\as\bs)}).
	 	\end{align*}
		\end{proof}

	    \begin{definition} 
	        The \textbf{conditional Neyman-Pearson learner} $\hat{h}^*$ is the active learner that makes the same sampling decision $\pi_{\hat{h}}$ as $\hat{h}$, and labels according to the following rules for each $\cC \in \Cr$. Conditional on the sample $S_\cC = (X^\cC_i, Y^\cC_i)_{i=1}^{n_\cC}$ in $\cC$, 
	        $$\hat{h}^*_n(x) = \left.\left(1 +  \argmax_{\sigma \in \{\pm 1\}} \prod_{i=1}^{n_\cC} P_{z_\cC = 1,\sigma_\cC = {\sigma}}(Y_i^\cC|X_i^\cC)\right)\right/2,$$
	        for all $x \in \cC$, where $P_{z_\cC, \sigma_\cC}(Y_i^\cC|X_i^\cC)$ is the probability of $Y_i^\cC$ given $X_i^\cC$, $z_\cC$ and $\sigma_\cC$.
	    \end{definition}
        \begin{proposition} \label{prop:np_optimal}
             Let $\hat{h}$ be any active learner, and $\hat{h}^*$ be the corresponding conditional Neyman-Pearson learner, then 
             $$\E_{\z,\sg}{\esd}\cE(\hat{h}_n) \ge \E_{\z,\sg}{\esd}\cE(\hat{h}^*_n).$$
        \end{proposition}
             \begin{proof}
                 We can decompose the excess risk as:
        	    \begin{align} \label{eqn:decomp}
        	        \cE(\hat{h}_n)  = \sum_{\cC \in \Cr} \cE_\cC(\hat{h}_n);
        	    \end{align} with $\cE_\cC(\hat{h}_n) \doteq \int_{\cC \cap \{\hat{h}_n \neq (1+\sigma_\cC)/2\}} |2\eta_{\z,\sg}(x) - 1| d P_X(x)$. Thus, we only need to show that for $\cC \in \Cr$, 
                \begin{align*}      \E_{S|\hat{h}} \E_{\z,\sg|S, \hat{h}} \cE_\cC(\hat{h}^*_n) \le  \E_{S} \E_{\z,\sg|S, \hat{h}} \cE_\cC(\hat{h}_n),
                \end{align*}
                where $\E_{S|\hat{h}}$ is the expectation taken over the distribution of $S$ given $\hat{h}$ and $\E_{\z,\sg|S, \hat{h}}$ is the taken over the conditional distribution of $(\z,\sg)$ given $S$ and $\hat{h}$. In the following proof, we suppress the dependency on $\hat{h}$ in notation for simplicity. Note that 
                    \begin{align*}
                        \E_{\z,\sg|S} \left[\cE_\cC(\hat{h}_n) | z_\cC = 0\right] & = 0; \text{ and } \\
                      \E_{\z,\sg|S} \left[\cE_\cC(\hat{h}_n) | z_\cC = 1, \sigma_\cC \right] &= 2 c_\eta r^{\as+d} \id\left(\hat{h}_n \neq (1+\sigma_\cC)/2\right).
                    \end{align*}
            Therefore,
            \begin{align*}
                \E_{S} \E_{\z,\sg|S} & \cE_\cC(\hat{h}_n)
                =  c_\eta r^{d + \as(\bs + 1)} - c_\eta r^{\as}  \E_S\bigg[\id \left(\hat{h}_n = 1\right)\\ & \quad \left(\p(z_\cC=1, \sigma_\cC = 1|S) - \p(z_\cC=1, \sigma_\cC = -1|S)\right)\bigg]
            \end{align*}
            is minimized if $\hat{h}_n(x) = 1$ when
            \begin{align*} 
                \frac{\p(z_\cC=1, \sigma_\cC = 1|S) }{\p(z_\cC=1, \sigma_\cC = -1|S)} \ge 1,
            \end{align*} and $\hat{h}_n(x) = 0$ otherwise. Finally, notice that 
            \begin{align*}
                \frac{\p(z_\cC=1, \sigma_\cC = 1|S) }{\p(z_\cC=1, \sigma_\cC = -1|S)} 
                 & = \frac{d P_{S|z_\cC=1, \sigma_\cC = 1}(S)}{d P_{S|z_\cC=1, \sigma_\cC = -1}(S)} \\
                 & = \frac{\prod_{i=1}^{n_\cC} P_{z_\cC=1,\sigma_\cC=1}(Y_i^\cC|X_i^\cC) } {\prod_{i=1}^{n_\cC} P_{z_\cC=1,\sigma_\cC=-1}(Y_i^\cC|X_i^\cC)}.
            \end{align*}
            where the last step is clear from the definition
            \begin{align*}
                d P_{S|z_\cC, \sigma_\cC}(S) = & \prod_{i=1}^{n} \pi_{\hat{h}}(X_i|\{X_j,Y_j\}_{j < i})  \\
                & \cdot \E_{\z_{(\cC)},\sg_{(\cC)}}  \prod_{\cC' \in \Cr}\prod_{i=1}^{n_{\cC'}} P_{z_{\cC'},\sigma_{\cC'}}(Y_i^{\cC'}|X_i^{{\cC}'}) d S
            \end{align*}
            Hence, the labeling decision of $\hat{h}^*$ minimize  $\E\limits_{\z,\sg}{\esd}\cE_\cC(\hat{h})$ for each $\cC$, hence $\E\limits_{\z,\sg}{\esd}\cE(\hat{h})$. 
        \end{proof}
    
    {\bf Notation:}
    For any distribution $P$ on $S$, we use $dP(S)/dS$ to denote the joint density of continuous $\{X_i\}_{i=1}^n$ and discrete $\{Y_i\}_{i=1}^n$.
    
    \begin{remark} \label{rmk:np} Proposition~\ref{prop:np_optimal} shows that we only need to lower-bound the excess risk rate for the collection of Neyman-Pearson classifiers. Further, since $\cE_\cC(\hat{h})$ is a function of $z_\cC, \sigma_\cC$ and $S_\cC$, we have $\E_{\z,\sg} \esds \cE_\cC(\hat{h}^*_n) = \E_{z_\cC,\sigma_\cC} \esdi \cE_\cC(\hat{h}^*_n)$ where $\esdi$ is the expectation over the distribution $\sdi$ of $S_\cC$ given $z_\cC, \sigma_\cC$ (where we have marginalized out the randomness in other cells). Furthermore, one can decompose $\sdi$ into the sampling location decision $\sdis$ and the labeling distribution $\sdil$: 
    \begin{align*}
        & d\sdi(S_c) \\ 
        = &\prod_{j=1}^{n_\cC} d \sdis(X_j^\cC|\{X_{i}^\cC,Y_{i}^\cC\}_{i  \le j}) \sdil(Y_j^\cC|X_j^\cC).
    \end{align*}
 
    \end{remark}
        \begin{proposition}\label{prop:NP}
             Let $\hat{h}^*$ be any conditional Neyman-Pearson learner. Then,
             \begin{align*}
                 \E_{\z,\sg} \esds \cE(\hat{h}^*_n) \ge C_1 n^{-\frac{\as(\bs+1)}{2\as+d}}.
             \end{align*}
             for some $C_1 > 0$.
             \begin{proof}
                    By \eqref{eqn:decomp} and the Remark~\ref{rmk:np}, we have
                    \begin{align*}
                        \E_{\z,\sg} \esds \cE(\hat{h}^*_n) = \sum_{\cC \in \Cr} \E_{z_\cC, \sigma_\cC} \esdi \cE_\cC(\hat{h}^*_n).
                    \end{align*}
                    Let $m \doteq  {r^d n}/{2} \equiv (c_\eta r^\as)^{-2} / 2$,
                    \begin{align*}
                        & \E_{z_\cC, \sigma_\cC} \esdi \cE_\cC(\hat{h}^*_n)  \\
                        \ge &  \E_{z_\cC, \sigma_\cC} \sum_{n_\cC=1}^{m} \esdi [\cE_\cC(\hat{h}^*_n)\ |\ 
                        |S_\cC| = n_\cC] \ \\
                        & \qquad \qquad \qquad \cdot \psdi(|S_\cC| = n_\cC) \\
                        \ge &  c_3 r^{d+\as} \E_{z_\cC, \sigma_\cC} \psdi\left( z_\cC = 1; |S_\cC| \le m\right),
                    \end{align*}
            where the last inequality by Lemma~\ref{lemma:risk_in_cell}. Furthermore, 
            \begin{align*}
                & \sum_{\cC \in \Cr}\E_{z_\cC, \sigma_\cC} \psdi\left( z_\cC = 1; |S_\cC| \le m\right) \\
                = \ & \sum_{\cC \in \Cr} \p (|S_\cC| \le m) \p \left(z_\cC = 1 | |S_\cC| \le m)\right) \\
                \ge \ & \frac{r^{\as\bs}}{1+c_4} \sum_{\cC \in \Cr}\p (z_\cC = 1 | |S_\cC| \le m)
                \ge \ \frac{r^{\as\bs-d}}{2(1+c_4)},
            \end{align*}
            where the second last inequality is due to Lemma~\ref{lemma:no_info}, and the last inequality is
            from the choice of $m$. Finally, 
            \begin{align*}
                \E_{\z,\sg} {\esd} \cE(\hat{h}^*_n) 
                & = \sum_{\cC \in \Cr}\E_{z_\cC, \sigma_\cC} \esdi \cE_\cC(\hat{h}^*_n) \\
                & = (c_3 r^{d+\as})\left(\frac{r^{\as\bs-d}}{2(1+c_4)}\right) \ge C_1 n^{-\frac{\as(\bs+1)}{2\as+d}},
            \end{align*}
            where $C_1 = \frac{c_3(\lambda^2/64)^{\frac{\as(\bs+1)}{2\as+d}}}{2(1+c_4)}>0$.
        \end{proof}
        \end{proposition}

\subsubsection{Supporting lemmas}

\begin{lemma}\label{lemma:dist_of_Y}
    Condition on $z_\cC$, $\sigma_\cC$ and $|S_\cC| = n_\cC$, ${\bf Y}_\cC = \{Y_j^\cC\}_{j=1}^{n_\cC} \overset{\tiny \text{i.i.d}}{\sim}$ Ber$(1/2 + z_\cC\sigma_\cC c_\eta r^\as)$.
            \begin{proof} 
            The conditional probability mass of ${\bf Y}_\cC$ is
                    \begin{align*}
                        & P_{{\bf Y}_\cC|z_\cC, \sigma_\cC,\hat{h}}({\bf Y}_\cC) \\
                        = & 
                        \frac{\prod_{j=1}^{n_\cC} d \sdis(X_j^\cC|\{X_{i}^\cC,Y_{i}^\cC\}_{i  \le j}) \sdil(Y_j^\cC|X_j^\cC)}{\prod_{j=1}^{n_\cC} d \sdis(X_j^\cC|\{X_{i}^\cC,Y_{i}^\cC\}_{i \le j})} \\
                        = & \prod_{j=1}^{n_\cC} \sdil(Y_j^\cC|X_j^\cC) 
                        =  \prod_{j=1}^{n_\cC} (1/2 + z_\cC \sigma_\cC Y_j^\cC c_\eta r^\as),
                    \end{align*}
                    which concludes the proof. 
            \end{proof}
\end{lemma}

\begin{lemma} \label{lemma:risk_in_cell}
    Let $n_\cC \le m = (c_\eta r^{\as})^{-2} / 2$ and $\hat{h}^*$ be a conditional Neyman-Pearson learner. Then, in cell $\cC$, for any combination of $(z_\cC, \sigma_\cC)$,
            $$\esdi [\cE_\cC(\hat{h}^*_n)\ |\ 
                        |S_\cC| = n_\cC] \ge c_3 r^{d+\alpha}\id(z_\cC = 1). $$
            for some $c_3 > 0$.
    \begin{proof}
        When $z_\cC = 0$, the inequality holds trivially. When $z_\cC = 1$, $$\cE_\cC(\hat{h}^*_n) = r^{d+\as} \id\left( \sigma_\cC\left[\frac{1}{n_\cC}\sum_{j=1}^{n_\cC} Y_j^\cC -\frac{1}{2}\right]<0  \right),$$
        the inequality holds by Lemma \ref{lemma:dist_of_Y} and the anti-concentration inequality (Lemma B.2).
    \end{proof}
\end{lemma}
        \begin{lemma}\label{lemma:no_info}
            Let $S_\cC = (X_j^\cC, Y_j^\cC)_{j=1}^{n_\cC}$ be such that $n_\cC = |S_\cC| \le m$. Then, 
            \begin{align*}
                \frac{\E\limits_{\sigma_\cC} d\sdif(S_\cC)} {\E\limits_{\sigma_\cC} d\sdib(S_\cC)} \le c_4,
            \end{align*}
            for some absolute constant $c_4>0$. Consequently, 
            $$\p \left(z_\cC = 1| |S_\cC| \le m\right)
                 \ge \frac{r^{\as\bs}}{1+c_4}.$$
            \begin{proof}
                By definition, 
                \begin{align*}
                    & \E_{\sigma_\cC} d\sdif(S_\cC)  
                      \\ = &\left(\frac{1}{2}\right)^{n_\cC}
                        \prod_{j=1}^{n_\cC} d\sdis\left(X_j^\cC|(X_{i}^\cC,Y_{i}^\cC)_{i  \le j}\right), \\
                    &\E_{\sigma_\cC} d\sdib(S_\cC)  \\
                     \ge &\frac{1}{2} \left(\frac{1}{2} + c_\eta r^\as \right)^{n_\cC/2}\left(\frac{1}{2} - c_\eta r^\as \right)^{n_\cC/2} \\  & \qquad \cdot \prod_{j=1}^{n_\cC} d\sdis\left(X_j^\cC|(X_{i}^\cC,Y_{i}^\cC)_{i  \le j}\right).
                \end{align*}
                Thus,
                \begin{align*}
                    \frac{\E\limits_{\sigma_\cC} d\sdif(S_\cC)} {\E\limits_{\sigma_\cC} d\sdib(S_\cC)} 
                    & \le 2 (1 - 4/m)^{-m/2}
                    \le  c_4,
                  \end{align*}
                for $c_4 = 16e^2$. Consequently,
                \begin{align*}
                    & \p \left(z_\cC = 1| |S_\cC| \le m\right) \\
                    = & \frac{\p(z_\cC = 1, |S_\cC| \le m)}{\p( |S_\cC| \le m)} \\
                    = &\frac{\p(z_\cC = 1, |S_\cC| \le m)}{\p( z_\cC = 1, |S_\cC| \le m) + \p( z_\cC = 0, |S_\cC| \le m)} \\
                    \ge & \frac{r^{\as\bs}}{1+c_4}.
                \end{align*}
            \end{proof}
        \end{lemma}

\subsection{Proof of Upper-bounds}
In this section, we establish the upper bounds on excess risk rates for Algorithm~\ref{alg:meta}. Due to space limit, we only outline the proof of the results under strong density condition and relegate the more direct proof under general density and other technical details in the supplementary materials. We start with a guarantee on the subroutine.

\begin{proposition}[Guarantees for Algorithm~\ref{alg:main}] \label{prop:excess_risk_bound_soft_margin} 

Let $n_0 \in \N$ and $\alpha\beta' \le d$. Let $\{S_\cC\}_{\cC \in r_0}$ be the outputs of Algorithm~\ref{alg:main} with input $n_0$, $\lambda$, $\alpha$ and $\delta_0\in (0,1)$, and $\hat{h}_{n_0,\alpha}$ be any classifier that satisfies $\hat{h}_{n_0,\alpha}(x)\in S_\cC,\forall x \in \cC \in \mathscrbf{C}_{r_0}$. Under Assumption~\ref{a:holder}~and~\ref{asmp:general_tsy} and strong density condition, with probability at least $1-\delta_0$, 
    \begin{align*}
        \mathcal{E} \left(\hat{h}_{n_0,\alpha}\right)
        \leq \ &C_5\left(
        \ee^{\frac{\alpha(\beta+1)}{2 \alpha+d}}  
        \left(\frac{\lambda^{\frac{d}{\alpha}} \log \left(\frac{4L \lambda^{2} n_0}{\delta_0}\right)}{ n_0}\right)^{\frac{\alpha(\beta+1)}{2 \alpha+d}} \right. \\
        + & \left.\left(\frac{ \lambda^{\frac{d}{\alpha}\vee\beta'} \log \left(\frac{4L \lambda^{2} n_0}{\delta_0}\right)}{ n_0}\right)^{\frac{\alpha(\beta'+1)}{2 \alpha+d-\alpha \beta'}} 
        \right)
    \end{align*}
for some constant $C_5 > 0$, which are independent of $n_0, \lambda, L, \ee$ and $\delta_0$. 
\end{proposition}
\begin{proof}
    Under some favorable event $\xi_\as$ with probability at least $1-\delta_0$ (Lemma A.1), the following holds: 
    
    $\bullet$ Algorithm~\ref{alg:main} will reach level 
        \begin{align*}
        r_\min &\le \max\left\{\left(\frac{c_7\lambda^{-2}\ee\log \left(\frac{4L\lambda^{2} n_0}{\delta_0}\right)}{n_0}   \right)^{\frac{1}{2\alpha+d}}, \right.\\
        & \quad \left.\left(\frac{c_7\lambda^{\beta'-2}\log \left(\frac{4L\lambda^{2} n_0}{\delta_0}\right)}{n_0}   \right)^{\frac{1}{2\alpha+d-\alpha\beta'}}  \right\}\\
        & \doteq \max\{Q_1, Q_2\};
    \end{align*}
    for some $c_7 > 0$ (Lemma A.4);
    
    $\bullet$ Algorithm~\ref{alg:main} never eliminates Bayes labels (Lemma A.2);
    
    $\bullet$ $\forall \cC \in \mathscrbf{C}_{r_0}, \forall x \in \cC, \forall y \in S_\cC$, $\eta_{(1)}(x) - \eta_y(x) \le 10 \lambda r_\min^\as$, and $S_\cC$ contains only Bayes labels in regions where $\ms > 10 \lambda r_\min^\as$ (Lemma A.3);
    
    When  $Q_1 \le Q_2$, $\ee + C_\bs r_\min^{\as\beta'} \le c_8 r_\min^{\as\beta'}$ for some $c_8>0$, 
    \begin{align*}
        \cE{(\hat{h}_{n_0,\alpha})} &\le P_X(\{x: \mh(x) \le 10\lambda r_\min^\as\}) (10\lambda r_\min^\as)\\
        & \le C'_5 \left(\frac{ \lambda^{\frac{d}{\alpha}\vee\beta'} \log \left(\frac{4L\lambda^{2} n_0}{\delta_0}\right)}{ n_0}\right)^{\frac{\alpha(\beta'+1)}{2\alpha+d-\alpha\beta'}},
    \end{align*}
    for some $C'_5 > 0$. When $Q_1 > Q_2$,
    \begin{align*}
        \cE{(\hat{h}_{n_0,\alpha})} 
        &\le P_X(\{x: \ms(x) \le 10\lambda r_\min^\as\}) (10\lambda r_\min^\as)\\
        & \le C''_5 \ee^{\frac{\alpha(\beta+1)}{2 \alpha+d}}  
         \left(\frac{\lambda^{\frac{d}{\alpha}} \log \left(\frac{4L\lambda^{2} n_0}{\delta_0}\right)}{ n_0}\right)^{\frac{\alpha(\beta+1)}{2 \alpha+d}}.
    \end{align*}
    for some $C''_5>0$. We then conclude the proof by choosing $C_5 = \max\{C_5',C_5''\}$.
\end{proof}


{\bf Outline of Proof for Theorem~\ref{thm:aggregation_strong_density} and \ref{thm:aggregation}.} The \emph{Correctness of aggregation} relies on the fact that Algorithm~\ref{alg:meta} a) never adds back removed labels, and b) stops aggregating labels when all labels are about to be removed from a cell --  this ensures the final candidate set $\mathcal{L}_\cC$ contains no bad labels and is non-empty. By Proposition~\ref{prop:excess_risk_bound_soft_margin}, we have  excess risk bounds for all $\alpha_i \le \alpha$, among which the largest one satisfies $\alpha-\alpha_{i} \leq 1/\lfloor\log(n)\rfloor^3$. Direct calculation shows that the excess risk bound with $\alpha_{i^*}$ is only a constant factor away from the one with $\alpha$.

\section{CONCLUSION}

In this paper, we have shown that simple nuances in notions of margin---seemingly having to do with uniqueness of the Bayes classifier---affect whether any active learner can gain over passive learning. Our main result is the lower bound (Theorem~\ref{thm:lower}), which requires proof techniques quite different from the usual lower bounds arguments in active learning, e.g. \cite{Minsker12}, \cite{Locatelli-et-al17}. We also show that savings remain possible in the worst case over $P_X$, and also under a refined margin condition in regimes with small sampling budget. 

Our main Theorem \ref{thm:lower} is shown here for the binary case, which does not distinguish between uniqueness of the Bayes and all labels being equivalent; as such it leaves open the possibility of a more refined picture in the case of multiple labels, i.e., whether allowing multiple labels (but not all) to be equivalent is enough to preclude savings over passive learning. 

Finally, while our results concern the nonparametric setting of active learning, it remains open whether similar nuances in achievable rates occur in parametric settings with bounded VC classes. 

\subsubsection*{Acknowledgements}
 The three authors, Samory Kpotufe, Gan Yuan, Yunfan Zhao, are listed in alphabetical order. Samory Kpotufe acknowledges support under a Sloan fellowship, and NSF Grant Id 1739809. He is also a visiting faculty at Google AI Princeton.

\newpage

\bibliography{refs.bib}


\clearpage
\appendix

\thispagestyle{empty}

\onecolumn \makesupplementtitle

\section{Proof of the Theorem 2 and 3}
To begin with, we define some quantities and notions that will be used in the lemmas. 

\begin{definition}
    \label{def:cell_loss}
    Let $A$ be any measurable subset of $\cX$ and $y \in [L]$. We define the regression function in $A$ for label $y$ as $\eta_y(A) \doteq \left[\int_A \eta_y(x) dx\right] / \left[\int_A dx\right]$. 
\end{definition}
Given $n_A$ independent samples $\{(X^A_j, Y_j)\}_{j=1}^{n_A}$ in $A$, an unbiased estimator of $\eta_y(A)$ is $$\hat{\eta}_y(A) \doteq \frac{1}{n_A}\sum_{i=1}^{n_A} \id(Y_i = y). $$

To get the high probability bound, we focus the discussion on a subset under which the estimation error of $\hat{\eta}$ at each cell is small throughout the proof. We consider a favorable event $\xi_\alpha \doteq \bigcap_{r \in \cI_r, \cC \in \Cr } \xi_{\cC, r, \alpha}$, where
\begin{align*}
    \cI_r & \doteq \{1/2, 1/4, \ldots, r_{\min}, r_{\min}/2\}, \\
    \xi_{\cC, r,\alpha} & \doteq \left\{\left\lVert\hat{\eta}(\cC)-\eta(\cC)\right\rVert_\infty \leq {\lambda r^\alpha}\right\}.
\end{align*}

The following lemma shows that $\xi_\alpha$ is indeed a high probability event. 

\vspace{0.2cm}
\begin{lemma}\label{lem:high_prob}
    $\mathbb{P}(\xi_\alpha) \geq 1- \delta_0$.
    \begin{proof}
    By Hoeffding's inequality, for each $y \in [L]$,
    $$
    \p\left(|\hat\eta_y(\cC)-\eta_y(\cC)|\geq \lambda r^\alpha\right)\leq \frac{ \delta_0 r^{d+1}}{L}.
    $$
    By union bound, $\p(\xi_{\mathscrbf{C},r, \alpha}) \geq 1 - \sum_{y = 1}^L (\delta_0 r^{d+1})/L = 1 - \delta_0 r^{d+1}$. Another application of union bound yields $\p(\xi_\alpha) \geq 1 -  \sum_{r\in \cI_r}r^{-d}  \delta_0 r^{d+1}\ge 1 - \delta_0$. 
    \end{proof}
\end{lemma}

Next, we show some desired properties of Algorithm 2 on the favorable event $\xi_\alpha$. In particular, Lemma~\ref{lem:non_active_region} shows that, Algorithm 2 never eliminate Bayes labels;  Lemma~\ref{lem:no_mistake_soft_margin} shows that Algorithm 2 predicts only Bayes labels in the area where soft margin is large enough; Lemma~\ref{lemma:r_min} shows that the algorithm will at least reach some certain level $r_\min$ of partition.  

\vspace{0.2cm}

\begin{lemma}\label{lem:non_active_region}
On the event $\xi_\alpha$, suppose that Algorithm 2 is in the depth that the partition is of sidelength $r$. For any $x \in \cX$, we have $\eta_y(x)< \eta_{(1)}(x)$ 
for any $y \not\in S_\cC$, where $x \in \cC \in \Cr$. That is, the algorithm never eliminate Bayes labels. 
    \begin{proof}
    For any $y \in  [L]$, by definition of $\xi_\alpha$ and smoothness assumption, we have 
    $$|\hat\eta_y(\cC)-\eta_y(\cC)|\leq ||\hat\eta(\cC)-\eta(\cC)||_{\infty}\leq \lambda r^\alpha;$$
    $$|\eta_y(x)-\eta_y(\cC)|\leq ||\eta(x)-\eta(\cC)||_{\infty}\leq \lambda r^\alpha.$$
     
    By the algorithm design,  $\hat\eta_{(1)}(x)-\hat\eta_y(x)\geq 6\lambda r^\alpha.$
    Therefore,
    $\eta_{(1)}(x)-\eta_y(x)\geq \hat\eta_{(1)}(\cC)-\hat\eta_y(\cC)-4\lambda r^\alpha > 0.$
    \end{proof}
\end{lemma}

\begin{lemma} \label{lem:no_mistake_soft_margin}
    On the event $\xi_{\alpha}$, suppose that Algorithm 2 is in the depth that the partition is of side length $r$. If 
    $\eta_{(1)}(x)-\eta_y(x)\geq \Delta_r=10\lambda r^\alpha$    
    for some $x \in \cX$ and $y \in [L]$, then for the cell $\cC \in \Cr$ that contains $x$, the label $y$ will be eliminated. Consequently, for any $x \in \cX$ with $\ms(x)> \Delta_r$, $S_\cC$ contains only Bayes labels. 
    
    \begin{proof}  
    For any $y \in [L]$, by Assumption~1, $\eta_{(1)}(x)-\eta_y(x)\geq \eta_{(1)}(\cC)-\eta_y(\cC)-2\lambda r^\alpha.$
	By the definition of $\xi_\alpha$,  we have 
	$|\eta_y(\cC)-\hat\eta_y(\cC)|\leq \lambda r^\alpha$, and hence
	\begin{align*}
	    \hat\eta_{(1)}(\cC)-\hat\eta_y(\cC)  \geq &
	    |\eta_{(1)}(\cC)-\eta_y(\cC)|-|\eta_y(\cC)-\hat\eta_y(\cC)|-|\eta_{(1)}(\cC)-\hat\eta_{(1)}(\cC)|\\
	    \geq & |\eta_{(1)}(\cC)-\eta_y(\cC)|-2\lambda r^\alpha\\
	    \ge & 6\lambda r^\alpha\,.
	\end{align*}
    
    \end{proof}
\end{lemma}

\begin{lemma}\label{lemma:r_min}
On the event $\xi_\as$,
\begin{enumerate}[i)]
    \item Under Assumption 1 and 2, then the finest partition Algorithm 2 can reach satisfies
    \begin{align*}
        r_\min \le \left(\frac{c_6\lambda^{-2}\log \left(\frac{4 L  \lambda^{2} n_0}{\delta_0}\right)}{n_0} \right)^{1/(2\alpha+d)};
    \end{align*}
    for some $c_6 > 0$;
    \item Under Assumption 1 and 2, and assume further that strong density condition holds for $c_d > 0$, then
    \begin{align*}
        r_\min \le \max\left\{\left(\frac{c_7\lambda^{-2}\ee\log \left(\frac{4L\lambda^{2} n_0}{\delta_0}\right)}{n_0}   \right)^{\frac{1}{2\alpha+d}},      \left(\frac{c_7\lambda^{\beta'-2}\log \left(\frac{4L\lambda^{2} n_0}{\delta_0}\right)}{n_0}   \right)^{\frac{1}{2\alpha+d-\alpha\beta'}}   \right\},
    \end{align*}
    for some $c_7 > 0$.
\end{enumerate}
\begin{proof}
\begin{itemize}
    \item[i)] The total budget is not sufficient for a finer partition than length $r_{\min}$, hence
    \begin{align*}
    n_0&\leq \sum_{r\in \cI_r} |\mathcal{A}_r|n_r  \leq \sum_{r\in \cI_r}r^{-d} \cdot \frac{2\log(2L/\delta_0 r^{d+1})}{\lambda^2 r^{2\alpha}} \\
    &\le \frac{2(d+1)\log 2}{\lambda^{2}}\log(2L/(\delta_0 r_{min}^{d+1})) \sum_{r\in \cI_r} r^{-(2\alpha+d)} \\
    &\le \frac{2(d+1)\log 2}{\lambda^{2}}\log(2L/(\delta_0 r^{d+1}_{\min})) \left(\frac{r_{\min}^{-(2\alpha+d)}4^{2\alpha+d}}{2^{2\alpha+d}-1}\right)\\
    &\le \frac{4(d+1)\log 2}{\lambda^{2}} \log(2L/(\delta_0 r^{d+1}_{\min}))\left(\frac{r_{\min}^{-(2\alpha+d)}4^{2\alpha+d}}{2\alpha+d}\right),
    \end{align*}
    
    where the last equality is from the inequality $2^u - 1 \ge \frac{u}{2}$ for $u\in \mathbb{R}^+$. 
    We now prove an upper bound on $\log(2L/(\delta_0 r^{d+1}_{\min}))$. Use the trivial bound
    $$
    \frac{\log \left(2L/(\delta_0 r^{d+1}_{\min})\right)}{2 \lambda^{2} r_{\min}^{2 \alpha}} = n_{r_\min} \leq n_0
    $$
    and $\delta_0 r_{\min}^{d+1}<\delta_0/2 \leq 2e^{-1}$, we have
    $$
    \frac{\log(L)}{2\lambda^2 r_{\min}^{2\alpha}}\leq n_0
    $$
    which implies
    $$r_{\min}\geq \left(\frac{\log(L)}{2\lambda^2 n_0}\right)^{1/2\alpha}
    $$
    and therefore
    \begin{align}\label{eqn:log_bound}
        \log(2L/(\delta_0 r^{d+1}_{\min}))&\leq \log\left(\frac{2L}{\delta_0}\left(\frac{2\lambda^2n_0}{\log(L)}\right)^{(d+1)/2\alpha}\right)
        \leq \frac{d+1}{2\alpha} \log\left(\frac{4 L\lambda^2 n_0}{\delta_0}\right)
    \end{align}
    With this upper bound on $\log(L/(\delta_0 r^{d+1}_{\min}))$, we now proceed to upper bound $r_{\min}$. Clearly,
    \begin{align*}
    n_0&\le
    \frac{c_6}{\lambda^{2}} \log \left(\frac{4 L  \lambda^{2} n_0}{\delta_0}\right) r_{\min}^{-(2\alpha+d)}
    \end{align*}
    where $c_6 =  \frac{4(d+1)^2 4^{2\alpha+d} \log 2}{2 \alpha(2\alpha + d)}$. Therefore, 
    $$
    r_{\min}\le \left(\frac{c_6}{\lambda^{2}n_0} \log \left(\frac{4 L  \lambda^{2} n_0}{\delta_0}\right)\right)^{1/(2\alpha+d)}.
    $$
    \item[ii)] From the strong density condition and Lemma~\ref{lem:no_mistake_soft_margin}, we have a tighter bound on the number of active cells:
    $$|\mathcal{A}_r| \leq \frac{\ee + C_\bs (6\lambda r^\alpha)^{\beta'}}{c_d r^d}.$$
    Using similar argument as in i), we have
    \begin{align*}
        n_0&\leq \sum_{r\in \cI_r} |\mathcal{A}_r|n_r \\
        &\leq \sum_{r\in \cI_r}\frac{\ee + C_\bs (6\lambda r^\alpha)^{\beta'}}{c_d r^d}\cdot \frac{2\log(2L/\delta_0 r^{d+1})}{\lambda^2 r^{2\alpha}} \\
        &\le \frac{4(d+1)\log 2}{c_d\lambda^2} \log(2L/(\delta_0 r^{d+1}_{\min})) \left(\ee\frac{r_{\min}^{-(2\alpha+d)}4^{2\alpha+d}}{2\alpha+d}+ C_\beta {(6\lambda)}^{\beta'}  \frac{r_{\min}^{-(2\alpha+d-\alpha\beta')}4^{2\alpha+d-\alpha\beta'}}{2\alpha+d-\alpha\beta'}\right)\\
        &\le c_7 \lambda^{-2}\log \left(\frac{4L  \lambda^{2} n_0}{\delta_0}\right)  \max\left\{\ee{r_{\min}^{-(2\alpha+d)}}, \lambda^{\beta'} {r_{\min}^{-(2\alpha+d-\alpha\beta')}}\right\}.     
    \end{align*}
    where $c_7 =  \frac{4(d+1)^2 4^{2\alpha+d} \log 2}{c_d \alpha{(2\alpha+d-\alpha\beta')}}\max\{1,C_\beta 6^{\beta'}\}$, and the last step is from \eqref{eqn:log_bound}. Therefore,
    \begin{align*}
        r_{\min}\le  \max&\left\{\left(\frac{c_7 \lambda^{-2}\ee\log \left(\frac{4L  \lambda^{2} n_0}{\delta_0}\right)}{n_0}   \right)^{\frac{1}{2\alpha+d}}, \left(\frac{c_7 \lambda^{\beta'-2}\log \left(\frac{4L\lambda^{2} n_0}{\delta_0}\right)}{n_0}   \right)^{\frac{1}{2\alpha+d-\alpha\beta'}}   \right\}.
    \end{align*}

\end{itemize}

\end{proof}
\end{lemma}

Now we prove rates for Algorithm 2. The proposition below is a generalized version of Proposition 5, and it includes rates under strong density condition.

\begin{proposition}[Guarantees for Algorithm 2] \label{prop:excess_risk_bound_soft_margin} 

Let $n_0 \in \N$ and $\alpha\beta' \le d$. Let $\{S_\cC\}_{\cC \in r_0}$ be the outputs of Algorithm 2 with input $n_0$, $\lambda$, $\alpha$ and $\delta_0 \in (0,1)$, and $\hat{h}_{n_0,\alpha}$ be any classifier that satisfies $\hat{h}_{n_0,\alpha}(x)\in S_\cC,\forall x \in \cC \in \mathscrbf{C}_{r_0}$. Under Assumption 1 and 2,
\begin{enumerate}[i)]
    \item With probability at least $1-\delta_0$, 
    \begin{align*}
        \mathcal{E} \left(\hat{h}_{n_0,\alpha}\right)
        \leq& \ C_4\left(\frac{ \lambda^{\frac{d}{\alpha}} \log \left(\frac{4L  \lambda^{2} n_0}{\delta_0}\right)}{ n_0}\right)^{\frac{\alpha(\beta+1)}{2 \alpha+d}}
    \end{align*}
    \item Suppose further that strong density condition holds with some $c_d > 0$, then with probability at least $1-\delta_0$,
        \begin{align*}
        \mathcal{E} \left(\hat{h}_{n_0,\alpha}\right)
        \leq \ &C_5\left(
        \ee^{\frac{\alpha(\beta+1)}{2 \alpha+d}}  
        \left(\frac{\lambda^{\frac{d}{\alpha}} \log \left(\frac{4L \lambda^{2} n_0}{\delta_0}\right)}{ n_0}\right)^{\frac{\alpha(\beta+1)}{2 \alpha+d}} + \left(\frac{ \lambda^{\frac{d}{\alpha}\vee\beta'} \log \left(\frac{4L \lambda^{2} n_0}{\delta_0}\right)}{ n_0}\right)^{\frac{\alpha(\beta'+1)}{2 \alpha+d-\alpha \beta'}} 
        \right)
    \end{align*}
\end{enumerate}
for some constant $C_4, C_5 > 0$, which are independent of $n_0, \lambda, L, \ee$ and $\delta_0$. 
\end{proposition}

\begin{proof}[\bf Proof of Proposition~\ref{prop:excess_risk_bound_soft_margin}]
\quad \\
\begin{enumerate}[i)]
\item On $\xi_\as$ with probability at least $1-\delta_0$, we have by Part i) of Lemma~\ref{lemma:r_min},
    \begin{align*}
        \Delta_{r_{\min}} &= 10 \lambda r_{\min}^\alpha \le 10\lambda  \left(\frac{c_6\log\left(\frac{4 L\lambda^2n_0}{\delta_0}\right)}{\lambda^2n_0}\right)^{\frac{\alpha}{2\alpha + d}} \le 10 \left(\frac{c_6\lambda^{\frac{d}{\alpha}}\log\left(\frac{4L\lambda^2n_0}{\delta_0}\right)}{n_0}\right)^{\frac{\alpha}{2\alpha + d}}.
    \end{align*}

By Lemma~\ref{lem:no_mistake_soft_margin}, the classifier $\hat{h}_{n_0, \alpha}$ makes no error at $\{x: \ms(x) > \Delta_{r_\min}\}$, and thus 
\begin{align*}
    \cE{(\hat{h}_{n_0,\alpha})} &\le \p_X(\ms(x) \le \Delta_{r_{\min}}) \cdot \Delta_{r_{\min}} \le C_\beta \Delta_{r_{\min}}^{\beta+1} \le C_4 \left(\frac{\lambda^{\frac{d}{\alpha}}\log\left(\frac{4L\lambda^2n_0}{\delta_0}\right)}{n_0}\right)^{\frac{\alpha(\beta+1)}{2\alpha + d}},
\end{align*} 
where $C_4 = C_\beta 10^{\beta+1} c_6^{\frac{\alpha(\beta+1)}{2\alpha + d}}$.
\item On $\xi_\as$ with probability at least $1-\delta_0$, we have by Part ii) of Lemma~\ref{lemma:r_min},
    \begin{align*}
    \Delta_{r_{\min}}  &\le 10 \max \left\{
    \ee^{\frac{\alpha}{2 \alpha+d}}  
    \left(\frac{c_7 \lambda^{\frac{d}{\alpha}} \log \left(\frac{4L\lambda^{2} n_0}{\delta_0}\right)}{ n_0}\right)^{\frac{\alpha}{2 \alpha+d}}, \qquad \left(\frac{c_7 \lambda^{\frac{d}{\alpha}\vee\beta'} \log \left(\frac{4L\lambda^{2} n_0}{\delta_0}\right)}{ n_0}\right)^{\frac{\alpha}{2\alpha+d-\alpha\beta'}} 
    \right\}\\
    & \doteq 10\max\{Q_1, Q_2\}.
    \end{align*}
    {\bf Case 1:  $Q_1 \le Q_2$}
    
    Under this case, it is clear that $\ee \le c_8 \Delta_{r_{\min}}^{\beta'}$ for some $c_8>0$. 
    
    Therefore, 
    \begin{align*}
        \cE{(\hat{h}_{n,\alpha})} &\le \p_X(\ms(x) \le \Delta_{r_{\min}}) \Delta_{r_{\min}}\\
        & \le \p_X(\ms'(x) \le \Delta_{r_{\min}}) \Delta_{r_{\min}} \\
        & \le C_\beta (\ee + \Delta_{r_{\min}}^{\beta'}) \Delta_{r_{\min}}\\
        &\le C_\beta(c_8+1) \Delta_{r_{\min}}^{\beta'+1}\\
        & \le C'_5 \left(\frac{ \lambda^{\frac{d}{\alpha}\vee\beta'} \log \left(\frac{4L\lambda^{2} n_0}{\delta_0}\right)}{ n_0}\right)^{\frac{\alpha(\beta'+1)}{2\alpha+d-\alpha\beta'}},
    \end{align*}
    where $C'_5 = C_\beta(c_8+1) 10^{\beta'+1} c_7^{\frac{\alpha(\beta'+1)}{2\alpha+d-\alpha\beta'}}$.
    
    {\bf Case 2:  $Q_1 > Q_2$}
    
    Under this case, 
    \begin{align*}
        \cE{(\hat{h}_{n_0,\alpha})} 
        &\le \p_X(\ms(x) \le \Delta_{r_{\min}}) \Delta_{r_{\min}}\le C_\beta \Delta_{r_{\min}}^{\beta+1} \le C''_5 \ee^{\frac{\alpha(\beta+1)}{2 \alpha+d}}  
         \left(\frac{\lambda^{\frac{d}{\alpha}} \log \left(\frac{4L\lambda^{2} n_0}{\delta_0}\right)}{ n_0}\right)^{\frac{\alpha(\beta+1)}{2 \alpha+d}},
    \end{align*}
    where $C''_5 = C_\beta 10^{\beta+1} c_7^{\frac{\alpha(\beta+1)}{2 \alpha+d}}$. Finally, set $C_5 = \max\{C'_5,C''_5\}$ and the desired result follows.
    \end{enumerate}
    \end{proof}

\begin{proof}[\bf Proof of Theorem 2 and 3]
\quad \\ Due to their similarity, we only prove Theorem 2, and omit the proof of Theorem 3. The bound is trivial for $\alpha< \frac{1}{\log (n)}$, since $n^{-\alpha}\ge n^{-1/\log (n)}\ge \frac{1}{e}$. Thus, we will consider $\alpha\ge \frac{1}{\log (n)}$. Let $\delta_0=\delta/\left(\lfloor \log(n)\rfloor^3\right)$ and $\alpha_i=i/\lfloor\log(n)\rfloor^3$ for $i\in[\lfloor\log(n)\rfloor^3]$, as defined in Algorithm 2. Let $i^*$ be the largest integer $ i \in [\lfloor\log(n)\rfloor^3]$ such that $\alpha_i\le \alpha$. By Lemma~\ref{lem:high_prob} and \ref{lem:non_active_region}, on $\xi_{\alpha_i}$  with probability at least $1-\delta_0$, we have
$$\forall \cC \in \mathscrbf{C}_{r_0},\forall x\in \cC, \argmax_y\eta_y(x)\in \bcl^{\as_i}_\cC$$

By a union bound, with probability at least $1-\lfloor\log(n)\rfloor^3\delta_0=1-\delta$, above holds jointly for all $i\le i^*$. Thus, with probability at least $1-\delta$,  
$$\forall \cC  \in \mathscrbf{C}_{r_0},\forall x\in\cC, \argmax_y \eta_y(x)\subseteq \cap_{i \le i^*} \bcl_\cC^{\as_i},$$
and hence $\cap_{i \le i^*} \bcl_\cC^{\alpha_i} \neq \emptyset$.  Therefore, $\bcl_\cC \subset \bcl^{\as_{i^*}}_\cC$ for any $\cC \in \mathscrbf{C}_{r_0}$. By proposition~\ref{prop:excess_risk_bound_soft_margin} and the fact that budget for each $\alpha_i$ is $n_{0}=\frac{n}{\lfloor\log (n)\rfloor^{3}}$, we have 
$$
\mathcal{E} \left(\hat{h}_{n}\right)\leq  C_5\left(
        \ee^{\frac{\alpha_{i^*}(\beta+1)}{2 \alpha_{i^*}+d}}  
        \left(\frac{\lambda^{\frac{d}{\alpha_{i^*}}} \log \left(\frac{4L \lambda^{2} n_0}{\delta_0}\right)}{ n_0}\right)^{\frac{\alpha_{i^*}(\beta+1)}{2 \alpha_{i^*}+d}} + \left(\frac{ \lambda^{\frac{d}{\alpha_{i^*}}\vee\beta'} \log \left(\frac{4L \lambda^{2} n_0}{\delta_0}\right)}{ n_0}\right)^{\frac{\alpha_{i^*}(\beta'+1)}{2 \alpha_{i^*}+d-\alpha_{i^*} \beta'}} 
        \right)
$$ 

It remains to argue that going from $\alpha_{i^*}$ to $\alpha$, we add at most a constant multiplicative factor to the excess risk bound. Notice that  
\begin{align*}
    & \frac{\as(1+\bs)}{2\as+d} - \frac{\as_{i^*}(1+\bs)}{2\as_{i^*}+d}
    \le \frac{1+\bs}{2\as \lfloor\log (n)\rfloor ^3} \le \frac{1+\bs}{2 \log^2(n)}\cdot\frac{\log^3(n)}{\lfloor \log (n)\rfloor^3}
\end{align*}
where the last step is due to $\alpha\geq \frac{1}{\log(n)}$. Similarly, 
\begin{align*}
\frac{\alpha(1+\beta')}{2\alpha+d-\alpha\beta'}-\frac{\alpha_{i^*}(1+\beta')}{2\alpha_{i^*}+d-\alpha_{i^*}\beta'}
\le & \frac{(1+\beta')(\alpha-\alpha_{i^*})(2\alpha+d)}{(2\alpha+d-\alpha\beta')^2}\\
\le & \frac{(1+\beta')(2\alpha+d)}{\log^3(n)(2\alpha+d-\alpha\beta')^2}\cdot\frac{\log^3(n)}{\lfloor \log (n)\rfloor^3}\\
\le & \frac{(1+\beta')(2\alpha+d)}{\log^3(n)(2\alpha)^2}\cdot\frac{\log^3(n)}{\lfloor \log (n)\rfloor^3} \\
\le & \frac{(1+\beta')(2+d)}{4\log^3(n) \alpha^2}\cdot\frac{\log^3(n)}{\lfloor \log (n)\rfloor^3} \\
\le & \frac{(1+\beta')(2+d)}{4\log(n)} \cdot\frac{\log^3(n)}{\lfloor \log (n)\rfloor^3}
\end{align*}

where the last step is due to $\alpha\geq \frac{1}{\log(n)}$. Therefore, for $n$ sufficiently large,
\begin{align*}
  \left(\frac{\log^3 (n) \lambda^{\frac{d}{\alpha_{i^*}}}  \log \left(\frac{4L  \lambda^{2} n}{\delta}\right)}{ n}\right)^{-\frac{\alpha(1+\beta)}{2\alpha+d}+\frac{\alpha_{i^*}(1+\beta)}{2\alpha_{i^*}+d}}&\le 2 e^{\frac{1+\bs}{2 \log(n)}}, \\
  \left(\frac{\log^3(n) \lambda^{\frac{d}{\alpha_{i^*}}\vee\beta'} \log \left(\frac{4L \lambda^{2} n}{\delta}\right)}{ n}\right)^{-\frac{\alpha(1+\beta')}{2\alpha+d-\alpha\beta'}+\frac{\alpha_{i^*}(1+\beta')}{2\alpha_{i^*}+d-\alpha_{i^*}\beta'}}&\le 2 e^{(1+\beta')(2+d)/4}
\end{align*}
and hence Theorem 2 holds with 
for $C_2 = 2e^{(1+\beta')(2+d)} C_5$.
\end{proof}

\section{Technical Lemmas for the Lower-bound}

\begin{lemma}[Chernoff bound]\label{lemma:chernoff}
 	Suppose $Y_1, \ldots, Y_m$ be independent random variables taking values in $\{0,1\}$ and $\bar Y = \left.\left(\sum_{i=1}^m Y_i\right)\right/m$. Then, for $\varepsilon > 0$,
 	$$\p\left(\bar Y \ge (1+\varepsilon) \E \bar Y\right) \le \exp\left(- m \varepsilon^2 \E \bar Y /3\right).$$
\end{lemma}

\begin{lemma}[Anti-concentration inequality]\label{lemma:anticoncentration}
	Let $Y_1, \ldots, Y_m \overset{\tiny \text{i.i.d.}}{\sim}\text{Ber}(1/2 + \delta)$ for some $0 < \delta < 1/2$. If $m \le \delta^{-2}/2$, then
	\begin{align*}
	    \p\left(\frac{1}{m}\sum_{j=1}^m Y_j < \frac{1}{2} \right) \ge c_3,
	\end{align*}
	for some absolute constant $c_3 > 0$.
	\begin{proof}
	It follows directly from Theorem 2 (ii) of \cite{Mousavi10}.
	\end{proof}
\end{lemma}

\end{document}